\documentclass[letterpaper, 10 pt, conference]{ieeeconf}  

\IEEEoverridecommandlockouts                              

\usepackage{graphicx} 
\usepackage{amsmath} 
\usepackage{amssymb}  
\usepackage{algorithm}
\usepackage{algpseudocode}
\let\labelindent\relax
\usepackage[shortlabels]{enumitem}
\usepackage{multirow}
\usepackage{cite}

\newtheorem{definition}{Definition}
\newtheorem{lemma}{Lemma}
\newtheorem{remark}{Remark}
\newtheorem{corollary}{Corollary}

\title{\LARGE \bf
Agent Utilities over Generalized Voronoi Regions and their Gradients
}

\author{Andre N. Costa$^{1,2}$, Petter \"Ogren$^{1}$ and Carlos H. C. Ribeiro$^{2}$
\thanks{$^{1}$Andre N. Costa and Petter \"Ogren are with the Robotics, Perception and Learning Lab., School of Electrical Engineering and Computer Science, Royal Institute of Technology (KTH), SE-100 44 Stockholm, Sweden, {\tt\small andrenc@kth.se, petter@kth.se}}
\thanks{$^{2}$Andre N. Costa and Carlos H. C. Ribeiro are with the Computer Science Division, Aeronautics Institute of Technology, Sao Jose dos Campos, SP, Brazil, {\tt\small carlos@ita.br}}%
}

\begin{document}

\maketitle
\thispagestyle{empty}
\pagestyle{empty}

\begin{abstract}
In this paper, we generalize the concept of Voronoi regions, define agent utility as the integral of a utility density over the corresponding Voronoi region, derive gradients of the utility, and illustrate the approach in a two-team example from soccer.

The generalization of Voronoi regions is in the form of so-called
Cost-Induced Voronoi (CIV) regions, where the agent state space may differ from the space being partitioned. One example of such regions is when the cost is given by the optimal solution of an LQR control problem. Then the agent states include position as well as velocity, while the partitioned space only includes positions.

The agent utility is defined by integrating some utility density over the CIV region of the agent.
This utility density might be the probability density of some beneficial event, such as receiving a pass in soccer. The utility is then the overall probability of receiving a pass and the gradient represents a way to improve that probability.
We show how this utility gradient can be computed using the Reynolds Transport Theorem from fluid mechanics, and that this approach achieves similar accuracy while reducing computation time by about an order of magnitude compared to a baseline finite-difference approximation.
\end{abstract}

\section{INTRODUCTION}
In dynamic and competitive environments, agents might reason about which regions of space they can reach or influence more efficiently than other agents. In the literature, such a region is sometimes called the agent’s dominance region, and one way to compute it is through the classical Voronoi diagram, where each agent is assigned the set of points for which it is closer than any other agent under a given metric.

In this paper, we generalize the Voronoi regions by separating the agent state space from the points in the region to be partitioned. The metric of classical Voronoi diagrams then becomes a cost function, since it is no longer a distance between elements of the same set. We call the resulting regions Cost-Induced Voronoi (CIV) regions.

\begin{figure}
  \centering
  \includegraphics[width=0.99\linewidth]{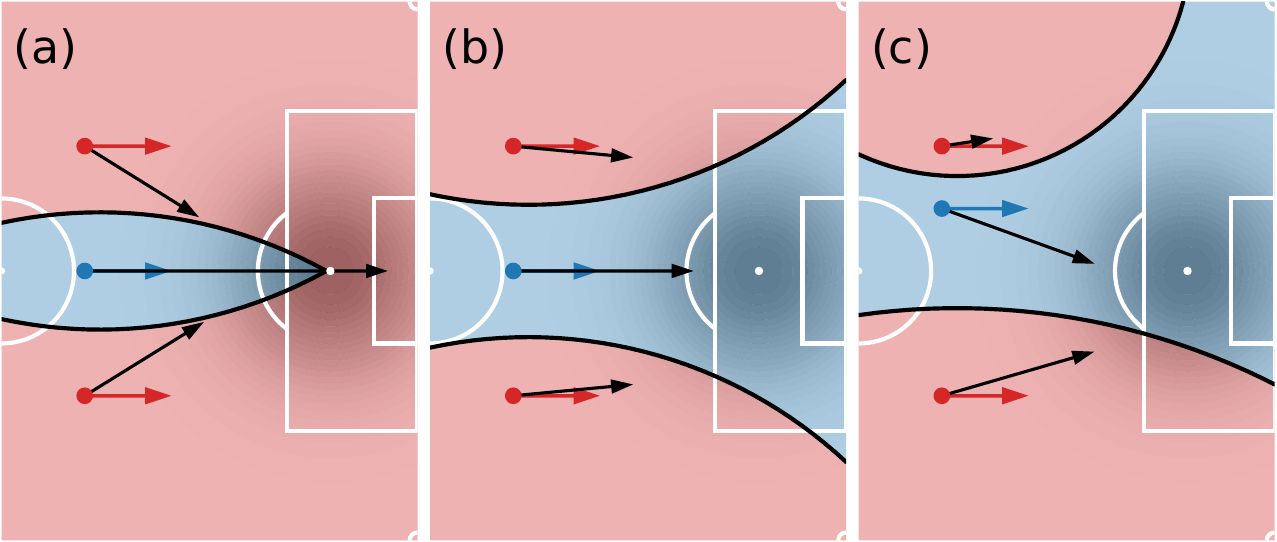}
  \caption{LQR-Cost-Induced Voronoi regions (LQR-CIV), and their gradients, for two teams on a soccer field. 
  Two red attackers and one blue defender are moving to the right (colored markers and arrows). The utility density for different parts of the field has a maximum in front of the goal, as this is considered the most attractive target for a pass from the attacking team.
  In (a) the blue team is physically weaker (higher drag and control cost, see Section \ref{sec:football}), resulting in smaller blue LQR-CIV regions. The utility gradients (black arrows) show that it would be beneficial for the red attackers to move
  closer to the weaker defender, to capture even more of the high utility area in the center.
  In (b) the blue team is physically stronger (lower drag and control cost), resulting in a larger blue LQR-CIV region. The utility gradients of the red attackers now suggest that it is best to maintain a separation to the strong defender while moving to the right.
  In (c) the blue team is still stronger, but the position of the blue agent is shifted upwards.
  Now the upper red attacker should move slightly away from the strong defender to increase separation, while the lower red attacker already has sufficient separation and should move toward the goal to gain more high-utility area.
  }
  \label{fig:casestudy}
\end{figure}

The costs can be arbitrary, but an important subclass arises when they are given by the solution of Linear Quadratic Regulator (LQR) problems; we call these LQR-CIV. In particular, we review the case where the linear system corresponds to a second-order dynamical system with drag, parameterized by the control cost and the drag coefficient. We show that this case gives rise to quadratic boundaries of the LQR-CIV regions, reflecting the underlying optimal-control cost structure.

We define the utility of each agent state as the integral of a utility density over its CIV region. The density may represent the probability of a beneficial event, such as scoring a goal or receiving a pass in sports. Since the CIV region corresponds to positions where the agent is more competitive (i.e., has the lowest cost), the integral of the utility density over the CIV region corresponds to the probability of the event occurring for that agent.

The gradient of the utility is an important tool for improving it, either in terms of finding the optimal agent states, or making incremental adjustments.
An illustration of utility gradients in a soccer example is shown in Fig.~\ref{fig:casestudy}.

Using Reynolds Transport Theorem (RTT) from fluid dynamics, we show how the utility gradient can be computed using a line integral along the CIV region boundaries instead of evaluating an area integral over the region.

We also show how the results extend to adversarial team settings, where the team utility is defined as the sum of the utilities of the non-adversarial agents.

To investigate the practical use of the result above, we compare a contour-integral (CI) method based on this result, and a baseline finite-difference (FD) method computing utilities under small perturbations, as in \cite{costaControlTheoreticFrameworkVoronoilike2025}.
Across representative configurations, CI matches FD within a few percent in the relative error of the computed utility gradients, while running about 25 times faster.

In summary, the main contributions are:
\begin{enumerate}
    \item A unified framework for CIV regions, in which partitions are defined by general costs based on agent states and spatial points (not necessarily in the same space), with LQR-CIV as a special case including both agent dynamics and control effort;
    \item A new boundary-integral expression for agent and team-utility gradients computed using RTT; and
    \item An efficient CI implementation based on the RTT result, compared with an FD baseline, reducing computation time by about 25× in our tests.
\end{enumerate}

\section{RELATED WORK}

Classical and centroidal Voronoi-type tessellations~\cite{cortesCoverageControlMobile2004,bulloDistributedControlRobotic2009,Bhattacharya2013MultiRobotCoverage,Lloyd1982LeastSquares,du1999centroidal} have been widely used for coverage, sensing, and task allocation, including density weighting to account for spatial importance. Extensions with anisotropic metrics~\cite{richter2015mahalanobis,hayashi20172d} and time-varying densities~\cite{Kennedy2019TimeVaryingDensity} further broaden applicability, yet these formulations remain distance-based and largely decoupled from agent dynamics and control effort. 

To incorporate dynamics and control effort, generalizations replace Euclidean distance with optimal-control costs. In~\cite{bakolas2010optimal,bakolas2013optimal,bakolas2015partitioning}, minimum-cost partitions under linear dynamics were studied, while quadratic boundaries and gradient formulas were derived in~\cite{costaControlTheoreticFrameworkVoronoilike2025}, linking them with density-weighted utilities. 

Adversarial Voronoi Regions (AVRs) were introduced in~\cite{costaOptimizingLocationsOpposing2025} to model team competition under Euclidean distance partitions. Related pursuit-evasion and interception works~\cite{zhouCooperativePursuitVoronoi2016,piersonInterceptingRogueRobots2017} extend this idea to reachability-based settings via safe-reachable sets (SRSs) under velocity-controlled dynamics; more specifically, \cite{zhouCooperativePursuitVoronoi2016} handles unequal speeds and nonconvex domains with grid-based Fast Marching computations, while \cite{piersonInterceptingRogueRobots2017} considers equal-speed agents where the SRS coincides with classical Voronoi cells, enabling decentralized pursuit.

This paper goes beyond the work above in the following sense:
The only previous work that investigates utility gradients is \cite{costaControlTheoreticFrameworkVoronoilike2025}, but it only considers classical Voronoi regions with straight line boundaries, not the general case considered here.

\section{BACKGROUND}
\label{sec:background}

In this section, we first describe classical Voronoi diagrams and their generalization to quadratic transfer costs. We then discuss adversarial team regions in competitive settings and, finally, introduce the RTT which will be a key tool when deriving team utility gradients over evolving partitions.

\subsection{Classical Voronoi Diagrams}

Classical Voronoi diagrams partition the Euclidean space into regions nearest to each agent in distance:
\begin{equation}
    V_i = \{\, q : \|q - p_i\| \leq \|q - p_j\|,\ \forall j \neq i \,\},
    \label{eq_voronoi}
\end{equation}
where \(p_i \in \mathbb{R}^n\) denotes the position of agent \(i\), and \(q \in \mathbb{R}^n\) a generic spatial point.

These regions are separated by planar bisectors orthogonal to the line connecting each pair of agents. Voronoi partitions have been widely used in coverage, sensing, and coordination problems~\cite{cortesCoverageControlMobile2004,bulloDistributedControlRobotic2009,Bhattacharya2013MultiRobotCoverage}, and serve here as the Euclidean baseline from which the general formulation is derived.

\subsection{Generalizing to Quadratic Transfer Costs}
\label{sec_general_quadratic_costs}
In many applications, the Euclidean distance in Equation~\eqref{eq_voronoi} is interpreted as a transfer cost—e.g., travel time or energy costs between two locations.
To generalize this concept, and include both Equations~\eqref{eq_voronoi} and a family of optimal control problems, it was suggested in \cite{costaControlTheoreticFrameworkVoronoilike2025} to study transfer costs for agent $\ell$ at state $x_\ell$ to a terminal state $x'$ of the form
\begin{equation}
    J_\ell(x') = (x'-\bar x_\ell)^T S_\ell(x'-\bar x_\ell)
               + c_\ell^T(x'-\bar x_\ell) + d_\ell,
\label{eq:transfer_cost}
\end{equation}
where $S_\ell \in \mathbb{R}^{m \times m}$,  $\bar x_\ell=\bar x_\ell(x_\ell)$ is a function of $x_\ell$, and the linear/constant
terms are $c_\ell\in\mathbb{R}^m$ and $d_\ell\in\mathbb{R}$. 
Here $m$ denotes the full state dimension, typically $m=2n$ for 
second-order dynamics (position and velocity), while $n$ is the spatial dimension.

It was shown in \cite{costaControlTheoreticFrameworkVoronoilike2025} that
the pairwise equal-cost manifold between agents $i$ and $j$, $B_{ij}=\{\,x': J_i(x')=J_j(x')\,\}$~\cite{costaControlTheoreticFrameworkVoronoilike2025}, is a quadratic hypersurface,
\begin{equation}
    x'^T M_{ij} x' + l_{ij}^T x' + b_{ij}=0,
\label{eq_quadratic_MLB}
\end{equation}

with coefficients
\begin{align*}
M_{ij} &= S_i - S_j,\\
l_{ij} &= 2(S_j\bar x_j - S_i\bar x_i)+(c_i-c_j),\\
b_{ij} &= \bar x_i^T S_i\bar x_i - \bar x_j^T S_j\bar x_j
          -c_i^T \bar x_i + c_j^T \bar x_j + (d_i-d_j).
\end{align*}

For any pair of agents $i \neq j$, the equal-cost condition defines, in general, a quadratic surface. 
Its geometric profile depends on the relation between the cost matrices $S_i$ and $S_j$.

When the matrices are equal ($S_i = S_j$), the quadratic terms cancel and the boundary becomes affine (a hyperplane in general, or a straight line in two dimensions). Its orientation depends on the common weighting $S_i$: if it is anisotropic, the separating plane is tilted according to the shared directional scaling. When this common weighting is isotropic ($S_i = S_j = \alpha I$), the plane is orthogonal to the line connecting the two agent centers, yielding the classical Voronoi diagram.

When the matrices differ ($S_i \neq S_j$), the boundary becomes a curved quadratic surface whose form—circular, elliptic, parabolic, or hyperbolic in two dimensions—depends on the eigenstructure of $M_{ij}=S_i-S_j$. Circular arcs arise when curvature is isotropic but nonzero, corresponding to unequal scalar weights ($S_i=\alpha_i I$, $S_j=\alpha_j I$). Elliptic boundaries occur when $M_{ij}$ is positive or negative definite, while parabolic and hyperbolic shapes appear when $M_{ij}$ is singular or indefinite.

In practice, for most agent costs, such as LQR-based controllers with isotropic costs, the matrices $S_i$ are symmetric positive definite and differ only slightly across agents. 
Consequently, the boundaries are straight, circular, or gently elliptic—capturing differences in agility or control cost—while parabolic and hyperbolic cases are less common.

\subsection{Adversarial Team Regions}
\label{sec:avr_background}

Voronoi methods have also been extended to competitive team settings \cite{costaOptimizingLocationsOpposing2025}. Given two opposing teams $T$ and $T'$, the Adversarial Voronoi Region (AVR) of team $T$ is the union of its members’ cells,
\cite{costaOptimizingLocationsOpposing2025}:
\[
\mathcal{R}_T = \bigcup_{i\in T} V_i .
\]

Further, the utility of owning a particular point (having it inside $\mathcal{R}_T$) for team $T$ is given by some utility density
 $\phi: \mathbb{R}^n \rightarrow \mathbb{R}$. This makes the aggregated team utility 
\begin{equation}
    H_T = \sum_{i\in T} \int_{V_i} \phi(q)\,dq.
\label{eq_classic_team_utility}
\end{equation}

This provides a quantitative measure of how favorable different team configurations are and opens up the possibility of moving agents along the gradient of $H_T$ to improve team utility.

\subsection{Reynolds Transport Theorem}
The Reynolds Transport Theorem (RTT) below generalizes the Leibniz rule to moving domains \cite{leal2007advanced}. 
For a domain $D(t)\subset\mathbb{R}^n$ with boundary $\partial D(t)$, outward unit normal $\hat n$, boundary velocity $v_b$, and a scalar field $F:\mathbb{R}^n \times \mathbb{R} \rightarrow \mathbb{R}^n$ defined over $D(t)$, the RTT \cite{leal2007advanced} states that

\begin{align}\label{eq:rtt}
\frac{d}{dt}\int_{D(t)} F(q,t)\,dV
&= \int_{D(t)} \tfrac{\partial F}{\partial t}(q,t)\,dV\nonumber\\ 
&+ \int_{\partial D(t)} F(q,t)\,(v_b^T \hat n)\,dS.
\end{align}

In what follows, we combine the cost-based partitioning framework introduced in
\cite{costaControlTheoreticFrameworkVoronoilike2025} and
\cite{costaOptimizingLocationsOpposing2025}, and extend it to the more general CIV formulation. We then use the RTT to derive analytical expressions for how agent utilities vary as the region boundaries move with the agents.  
In this transition, the AVR team utility $H_T$, in (\ref{eq_classic_team_utility}) defined over classical Voronoi regions, becomes a dynamics-aware team utility $U_T$ in (\ref{eq:team_utility}) defined over CIV regions. 

\section{PROPOSED APPROACH}
To provide a motivating special case for the CIV formulation below, we first derive transfer costs for a double integrator with linear drag and show they are quadratic, inducing curved, non-Euclidean Voronoi boundaries. We then define CIV regions in position space for general transfer costs and, using the RTT, obtain boundary integral formulas for the gradients of agent and team utilities.

\subsection{LQR cost for Double Integrator with Drag}
\label{sec:lqr_drag}

\begin{lemma}[LQR transfer cost]\label{lem:lqr_cost}
    Consider the motion model of a planar double integrator with linear drag,
\begin{equation}
    \ddot p = u - a v, 
\end{equation}
where the position, velocity, and control input (acceleration) are
$p,v,u \in \mathbb{R}^2$, respectively, and $a \in \mathbb{R}_+$ is the linear drag coefficient. Let $x = (p,v)$ be the combined state, and the transfer cost of some trajectory $x(t)$  with corresponding control $u(t)$ from $x_0$ to the origin be given by
\begin{equation}
    J(x_0)=\int_0^\infty \bigl(x(t)^T x(t) + r u(t)^T u(t)\bigr)\,dt,
    \label{eq:cost_lqr}
\end{equation}
with $r\in \mathbb{R}_+$ determining the tradeoff between control effort and state error. Then the corresponding optimal transfer cost is
\begin{equation}
    J^*(x) 
    = k_p \|p\|^2 + 2 k_{pv} v^T p + k_v \|v\|^2,
    \label{eq:lqr_J}
\end{equation}
where
\begin{align}
    & k_p = a \sqrt{r} + \tfrac{k_v}{\sqrt{r}},\nonumber \\
    & k_{pv} = \sqrt{r},\nonumber\\
    & k_v = -a r + \sqrt{a^2 r^2 + r(2\sqrt{r}+1)},\nonumber
\end{align}

Since the dynamics is translationally invariant with respect to $p$,  the optimal transfer cost to an arbitrary terminal position $q \in \mathbb{R}^2$ is
\begin{equation}
    J^*(q,x) 
    \;=\; k_p \|p-q\|^2 + 2 k_{pv}\, v^T (p-q) + k_v \|v\|^2,
    \label{eq:lqr_J_pprime}
\end{equation}
assuming deviations from $q$ are penalized in (\ref{eq:cost_lqr}).

The pairwise equal-cost manifold between two agents $x_i, x_j$, i.e., $\{q: J^*(q,x_i)=J^*(q,x_j)\}$ is given by a quadratic hypersurface.
\end{lemma}
\begin{proof}
First, (\ref{eq:lqr_J}) follows from the solution of the algebraic Riccati equation for the infinite-horizon LQR problem. 
Due to the rotational symmetry of the planar dynamics and cost, the Riccati matrix $P$ has the block structure
\[
P =
\begin{bmatrix}
k_p I_2 & k_{pv} I_2 \\
k_{pv} I_2 & k_v I_2
\end{bmatrix},
\]
which yields the coefficients $k_p$, $k_{pv}$, and $k_v$.

Then (\ref{eq:lqr_J_pprime}) is a pure translation of (\ref{eq:lqr_J}) in position, since the motion model is translationally invariant with regards to position.
Thus, the cost is on the form of (\ref{eq_quadratic_MLB}) and the quadratic hypersurface boundaries follow.
\end{proof}

\begin{remark}
By varying the parameters $r$ (control cost) and $a$ (drag) we have a fairly flexible family of agent models taking initial velocity into account, while having a simple closed form cost.
One can also use a finite time version of LQR theory, but that leads to increased complexity in terms of a time dependence in the $P$-matrix.
\end{remark}

The model above is used as a special case (LQR-CIV) of the general 
CIV framework described below.

\subsection{Cost-Induced Voronoi Regions}

We now present a generalization of the Voronoi regions, where the agent state spaces are separated from the position space.

\begin{definition}[Cost-Induced Voronoi Regions]
    Let there be $N$ agents, each specified by a parameter vector $x_i\in \mathbb{R}^m$, and let $d_i: \mathbb{R}^n \times \mathbb{R}^m \rightarrow \mathbb{R}$ denote a differentiable transfer cost for agent $i$ in state $x_i$ to reach a position $q \in \mathbb{R}^n$.
    For each agent $i$, define its \emph{cost advantage function}
    \begin{equation}
    g(q, x_i) := \min_{j \neq i} \bigl[d_j(q, x_j) - d_i(q, x_i)\bigr],
    \end{equation}
    which is positive wherever the transfer of agent $i$ is less costly than all competitors.  
    The corresponding Cost-Induced Voronoi (CIV) region is then    

\begin{align}
    G_i(x_i) 
    &= \bigl\{q \in \mathbb{R}^n \mid d_i(q,x_i) \le d_j(q,x_j), \ \forall j \ne i \bigr\} \label{eq:gen_voronoi_definition}\\
    &= \bigl\{q \in \mathbb{R}^n \mid g(q,x_i) \ge 0\bigr\},
\end{align}
i.e., the set of all points where agent $i$ has a nonnegative cost advantage.
\end{definition}

\begin{remark}
If $n = m$ and $d(q,x) = \|q - x\|$, the regions $G_i$ reduce to the classical Euclidean Voronoi cells. 
\end{remark}

\begin{definition}[LQR-CIV]
Let the LQR-CIV regions be given by $G_i(x_i)$, when the cost is
$d(q,x)=J^*(q,x)$ from Equation (\ref{eq:lqr_J_pprime}),
with $q \in \mathbb{R}^2$ and 
    \(x_i = (p_i, v_i) \in \mathbb{R}^4\).
\end{definition}

Examples of LQR-CIV regions can be seen in Fig.~\ref{fig:casestudy} and Section~\ref{sec:results}.
Following the discussion in Section \ref{sec_general_quadratic_costs},
in general, $G_i$ have quadratic boundaries. But 
when all agents share the same drag/cost parameters $p,r$ we get ($S_i = S_j$), 
and the pairwise boundaries become straight lines.

\subsection{The Agent Utility and its Gradient}

\begin{definition}[Agent utility]
Given a spatial utility density $\phi: \mathbb{R}^n \rightarrow \mathbb{R}$,  
the agent utility is defined as
\begin{equation}
U_i(x_i) = \int_{G_i(x_i)} \phi(q)\,dq,
\end{equation}
i.e., the total accumulated utility over its CIV region $G_i(x_i)$.
\end{definition}

In order to optimize the agent utility, we are interested in the gradient $\nabla_{x_i} U_i(x_i)\equiv \frac{\partial U_i(x_i)}{\partial x_i}$. To compute it in Lemma~\ref{lem_gradient} below, we first note that $G_i(x_i)$ varies with the agent state through its moving boundary. We therefore first relate boundary motion to changes in $x_i$.

\begin{lemma}[Boundary velocity]
\label{lem_boundary_velocity}
Let an agent follow a trajectory $x_i(t)\in\mathbb{R}^m$, and let 
$D(t)=G_i(x_i(t))\subset\mathbb{R}^n$ denote its CIV region at time $t$.
Consider a point $q(t)$ on the moving boundary $\partial D(t)$, satisfying
$g(q(t),x_i(t))=0$ for all $t$.
Let $v_b=\dot q$ be the boundary velocity and $\hat n_i$ the outward normal of $D(t)$.
Then,
\begin{equation}
    \hat n_i(q) = -\,\frac{\nabla_q g(q,x_i)}{\|\nabla_q g(q,x_i)\|}, 
    \quad
    v_b^T \hat n_i
    = \frac{\nabla_{x_i} g(q,x_i)^T\,\dot x_i}{\|\nabla_q g(q,x_i)\|}. \label{eq:bound_vel}
\end{equation}
\end{lemma}
\vspace{3mm}
\begin{proof}
First, $D(t)$ is given by $g\geq 0$, thus the gradient $\nabla_q g$ points inwards. Multiplying by $-1$ and normalizing we get the left part of (\ref{eq:bound_vel}).
Next, note that we have $g(q(t),x_i(t))=0$ along the boundary. Differentiating w.r.t. time gives
\[
0=\nabla_q g(q,x_i)^T\,\dot q + \nabla_{x_i} g(q,x_i)^T\,\dot x_i.
\]
With $v_b=\dot q$ and $\hat n_i=-\nabla_q g/\|\nabla_q g\|$,
we get
$$v_b^T \hat n_i=\hat n_i^T v_b = -\nabla_q g^T \dot q/\|\nabla_q g\|= \nabla_{x_i} g^T\,\dot x_i/\|\nabla_q g\|,$$ where we have used that the transpose of a scalar is a scalar.

\end{proof}

\begin{remark}[Interpretation of Lemma~\ref{lem_boundary_velocity}]
\label{rem:boundary_sensitivity}

We will look at two simple examples of boundary motion in terms of $v_b^T \hat n_i$.
In both examples $\mathbb{R}^n=\mathbb{R}^m$, i.e., the spatial point $q$ and the agent state $x_i$ have equal dimension.

First, consider Fig.~\ref{fig:boundaries}(a).
We have a circular boundary of $G_i$ resulting from $d_i(q,x_i)=\|q - x_i\|^2$ and $d_j(q,x_j)=C, \ \forall j\neq i$ constant.
Clearly, $\nabla_{x_i} g(q,x_i)$ points radially outward from the circle. Consider an agent motion $\dot x_i$ to the north (blue arrow). The scalar product $\nabla_{x_i} g(q,x_i)^T \dot x_i$ (the boundary motion given by (\ref{eq:bound_vel})) is then positive in the direction of the motion (north), zero in the orthogonal directions (east/west), and negative in the opposite direction. Thus the circle follows the agent, as expected (dashed circle in the figure).

Second, consider Fig.~\ref{fig:boundaries}(b).
Here we have classical Voronoi regions, where all costs are 
$d_i(q,x_i)=\|q-x_i\|^2,\ \forall i$.

The motion of the boundary depends on the direction of the agent velocity $\dot x_i$.
If $\dot x_i$ is orthogonal to the boundary $\partial D$, as indicated by the blue arrow,
the boundary translates uniformly to the right, as illustrated by the blue dashed line.
If instead the agent moves tangentially along the boundary, as indicated by the green arrow,
points of the boundary in the direction of motion move outward while points in the opposite direction move inward.
Thus the boundary pivots around the point of the line closest to $x_i$, as illustrated by the green dashed line.

The analytical expression for the boundary velocity will be derived later in
(\ref{eq:gamma}) and (\ref{eq:gamma_classical_voronoi}).
\end{remark}

\begin{figure}
  \vspace{5pt} 
  \centering
  \includegraphics[width=0.9\linewidth, trim={0 12cm 0 0},clip]{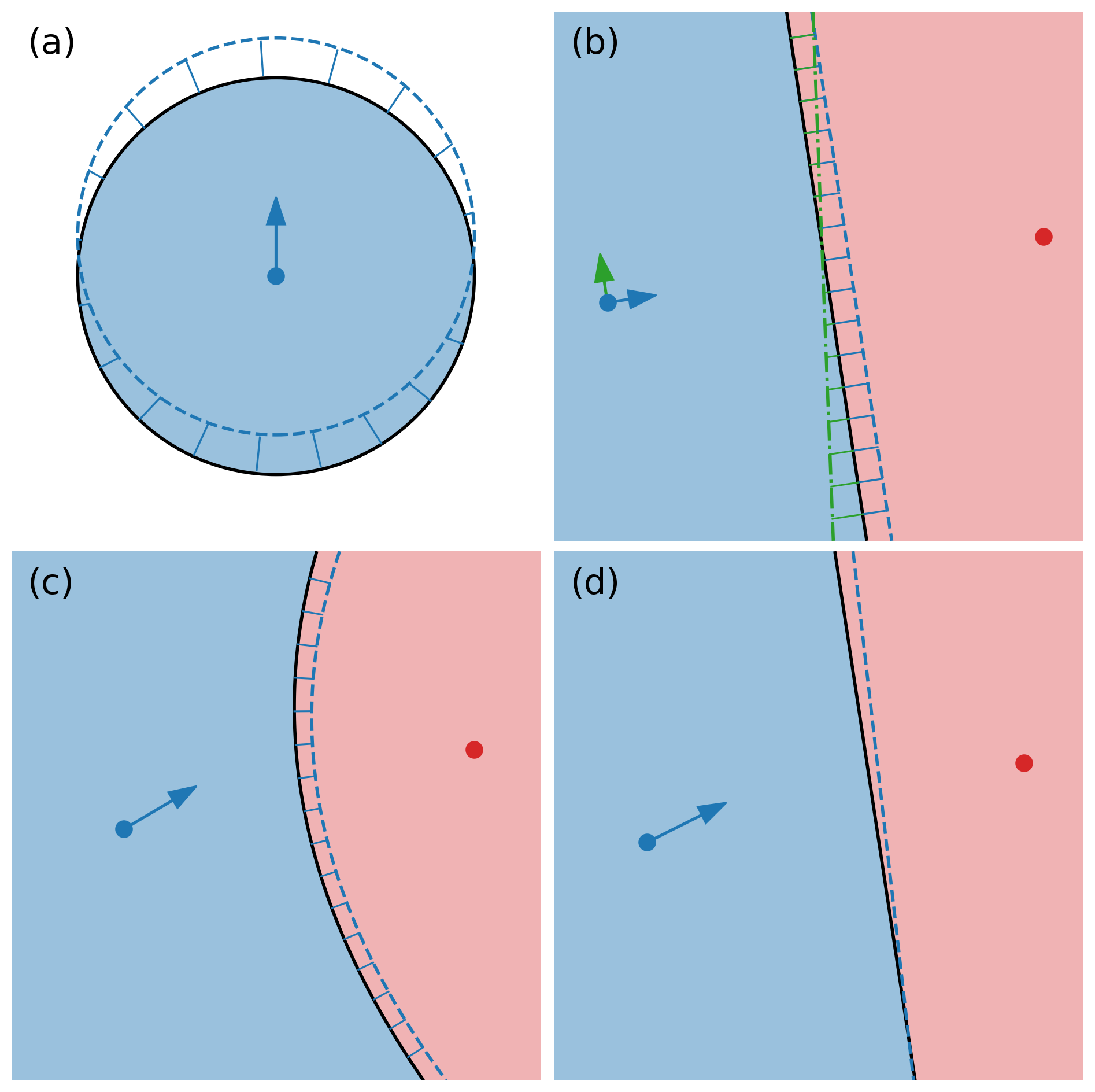}
  \caption{
  Illustration of CIV boundary shifts $v_b^T \hat n_i$ as a function of $\dot x_i$, in Equation (\ref{eq:bound_vel}) for two special cases, described in Remark \ref{rem:boundary_sensitivity}.
  (a) If all other agents have constant costs, while one has radially increasing, the boundary $\partial D$ is circular, following the agent $x$.
  (b) For classical Voronoi boundaries, $\partial D$ moves uniformly (translates) when the agent moves towards the other agent (blue arrow and lines), but pivots if the agent moves along the boundary (green arrow and lines).}
  \label{fig:boundaries}
\end{figure}

Having established how the boundary moves with the agent state, we now use the RTT to derive how this motion affects the agent utility.

\begin{lemma}[Utility gradient]
\label{lem_gradient}
Let $D(t) = G_i(x_i(t)) \subset \mathbb{R}^n$. If the density $\phi(q)$ is time-invariant, then
\begin{equation}
    \nabla_{x_i} U_i(x_i)
    \;=\; \int_{\partial D(t)} \phi(q)\, \gamma_i(q)\, dS,
\end{equation}
where
\begin{equation}
    \gamma_i(q) = \frac{\nabla_{x_i} g(q,x_i)}{\|\nabla_q g(q,x_i)\|}.
    \label{eq:gamma}
\end{equation}
\end{lemma}
\vspace{3mm}
\begin{proof}
Let the agent state evolve along a trajectory $x_i(t)\in\mathbb{R}^m$.  
By the RTT in Equation~\eqref{eq:rtt}, the time derivative of its utility is given by
\begin{equation}
    \frac{d}{dt} U_i(x_i(t))
     = \int_{D(t)} \tfrac{\partial \phi}{\partial t}\, dV
       + \int_{\partial D(t)} \phi(q)\,(v_b^T \hat{n}_i)\, dS.
    \label{eq:rtt_expression}
\end{equation}

Since $\phi$ is time-invariant, the first term vanishes.  
Substituting the boundary velocity from Lemma~\ref{lem_boundary_velocity}
into Equation~\eqref{eq:rtt_expression} yields
\begin{align}
    \frac{d}{dt} U_i(x_i(t))
    &= \nabla_{x_i}U_i(x_i)^T\,\dot{x}_i \nonumber\\
    &=
    \Biggl(
        \int_{\partial D(t)} 
        \phi(q)\,
        \frac{\nabla_{x_i} g(q,x_i)}{\|\nabla_q g(q,x_i)\|}\, dS
    \Biggr)^T\dot{x}_i.
    \label{eq:last_boundary_form}
\end{align}

Since $\dot{x}_i$ is arbitrary, the term in parentheses must equal 
$\nabla_{x_i} U_i(x_i)$, proving the result.
\end{proof}

\begin{corollary}[Classical Voronoi case]
If $\mathbb{R}^n=\mathbb{R}^m$ and the transition cost is given by
$d(q,x)=\|q-x\|^2$ (squared form, which yields the same
partitions as $\|q-x\|$ but simplifies the algebra), the  factor $\gamma$ in
Lemma~\ref{lem_gradient} reduces, on the common boundary
$\partial V_i\cap\partial V_j$, to
\begin{equation}
    \gamma_i(q)=\frac{q-x_i}{\|x_i-x_j\|}.
    \label{eq:gamma_classical_voronoi}
\end{equation}
\end{corollary}
\vspace{3mm}
\begin{proof}
Consider the boundary separating $V_i$ and $V_j$, defined by the level set
\begin{align}
    g(q,x_i) = d(q,x_j) - d(q,x_i) 
             = \|q - x_j\|^2 - \|q - x_i\|^2 = 0.
\end{align}
Then, the gradients are
$\nabla_{x_i} g = 2(q - x_i)$, 
$\nabla_{q} g   = 2(q - x_j) - 2(q - x_i) = 2(x_i - x_j)$.
Substituting into \eqref{eq:gamma} yields
\[
    \gamma_i(q) = \frac{2(q - x_i)}{\|2(x_i - x_j)\|}
                = \frac{q - x_i}{\|x_i - x_j\|}.
\]
\end{proof}

\begin{remark}
This result confirms that Lemma~\ref{lem_gradient} includes the classical Voronoi formulation as a special case. 
Specifically, for the Euclidean cost $d(q,p) = \|q - p\|^2$ and a constant so-called revenue factor $f \equiv 1$, 
the boundary expression reduces to the gradient formulation of classical Voronoi regions 
previously derived in Lemma~4.12 of~\cite{costaOptimizingLocationsOpposing2025}.
\end{remark}

Now we divide all agents into two adversarial teams, to investigate the combined team utility, and the gradient of that utility, as was done in \cite{costaOptimizingLocationsOpposing2025}, reviewed in Section~\ref{sec:avr_background}, but this time with the more general CIV regions.

\begin{definition}[Team utility]
Let $\phi : \mathbb{R}^n \rightarrow \mathbb{R}$ denote a spatial utility density,
and let $T \subset \{1, \ldots, N\}$ be the team of interest.  
The team utility is defined as the sum of the individual agent utilities within the team:
\begin{equation}
    U_T(X_T)
    = \sum_{i \in T} \int_{G_i(x_i)} \phi(q)\,dq, \label{eq:team_utility}
\end{equation}
where $X_T = \{x_i : i \in T\}$ denotes the set of agent states.
\end{definition}

Note that $\mathcal R_T = \bigcup_{i\in T} G_i(x_i)$ generalizes the AVR from Section~\ref{sec:avr_background}: here each $G_i$ is defined by a general cost rather than Euclidean distance, making $\mathcal R_T$ a dynamics-aware team partition. Given the team utility, we now consider how it varies with the collective state $X_T$ of the agents in the team.

\begin{lemma}[Gradient of Team Utility]
Let $L=\partial\mathcal R_T\cap\partial\mathcal R_{T'}$ denote the inter-team boundary.
Then, for any agent $i\in T$,
\begin{equation}
    \nabla_{x_i} U_T
    \;=\;
    \int_{\partial G_i \cap L}
    \phi(q)\,
    \frac{\nabla_{x_i} g(q,x_i)}
         {\|\nabla_q g(q,x_i)\|}\,dS . \label{eq:team_utility_gradient}
\end{equation}
\end{lemma}
\vspace{3mm}
\begin{proof}
By Lemma~\ref{lem_gradient}, $\nabla_{x_i}U_i$ depends only on inter-agent boundaries.
For agents $i,j\in T$, (in the same team) the contributions across their shared boundary
cancel because of opposite normal orientations.  
Thus only boundaries separating $T$ and $T'$ contribute to the total variation of $U_T$, leading to the expression above.
\end{proof}

Together, these results establish how partition boundaries induced by general transfer costs determine the evolution of agent and team utilities. 

\section{NUMERIC COMPUTATIONS AND EXAMPLE}
\label{sec:results}

In this section, we first compare numerically computing the utility gradient using the new result in Equation (\ref{eq:team_utility_gradient}) with a baseline using a finite-difference approximation and Equation (\ref{eq:team_utility}).
Then we look at the soccer example in Fig.~\ref{fig:casestudy}. 

\subsection{Computing the utility gradient}
\label{sec:verification}

Two different methods were implemented and compared: a finite-difference baseline that re-evaluates the team utility under small perturbations, and a contour-integral approach that applies the derived expression explicitly along the estimated boundary. 

\subsubsection{Finite-difference (FD) baseline}
The baseline method applies a uniform grid over $\mathbb{R}^n$ and approximates the Voronoi regions $G_i$ by computing (\ref{eq:gen_voronoi_definition}) for each grid cell.
Then the team utility can be approximated by summation using (\ref{eq:team_utility}).
Finally, to estimate the team utility gradient $\nabla_{x_i} U_T$, we can do a finite-difference approximation of the gradient, by making small perturbations of $x_i$, computing the differences in team utility, and dividing by the size of the perturbations.

\subsubsection{Contour-integral (CI) method}
A straightforward way of using the results in (\ref{eq:team_utility_gradient}) is to first estimate $G_i$ using a grid as above, and find an approximation to the boundaries $\partial G_i \cap L$ directly on this grid. The boundary is extracted as the set of grid edges whose neighboring cells belong to different teams, forming a closed contour that represents the inter-team boundary. 
Finally, we approximate the integral in (\ref{eq:team_utility_gradient}) by aggregating the values of the integrand along the approximated boundary.

\subsubsection{Comparison}
Table~\ref{tab:verification} summarizes gradients and runtimes across the three configurations in Fig.~\ref{fig:casestudy}. 
In all cases, the CI implementation produced gradients that closely matched those of the finite-difference method, with relative differences below $2.5\%$ for both position and velocity components.
The computation times were quite different, with FD requiring approximately $1.7$~s per configuration, while CI completed in under $80$~ms, thus achieving comparable accuracy at roughly 25 times lower runtime. 

\begin{table}
\vspace{4pt} 
\centering
\caption{Comparison of Finite-difference (FD) and contour-integral (CI) across the configurations of Fig.~\ref{fig:casestudy}.}

\label{tab:verification}
\renewcommand{\arraystretch}{1.1}
\setlength{\tabcolsep}{4pt}
\begin{tabular}{lcccc}
\hline
\textbf{Case} & \textbf{Method} & \textbf{Time [ms]} & $\|\nabla_p U_T\|$ & $\|\nabla_v U_T\|$ \\
\hline
\multirow{2}{*}{\shortstack[l]{(a) Striker\\advantage}} 
 & CI &  56.8  & 4.55 & 2.01 \\
 & FD & 1626.8 & 4.49 & 2.00 \\[1ex]
\multirow{2}{*}{\shortstack[l]{(b) Defender\\advantage}} 
 & CI &  74.4  & 3.29 & 1.44 \\
 & FD & 1604.2 & 3.28 & 1.45 \\[1ex]
\multirow{2}{*}{\shortstack[l]{(c) Asymmetric\\offset}} 
 & CI &  68.9  & 2.90 & 1.27 \\
 & FD & 1709.5 & 2.90 & 1.28 \\
\hline
\end{tabular}
\end{table}

\subsection{Case Study: Two Strikers vs. One Defender on a Half Football Pitch}
\label{sec:football}

We consider the defensive scenario on a half football pitch shown in Fig.~\ref{fig:casestudy}, where two red strikers attack and one blue defender protects the region near the penalty spot. All agents use LQR-CIV costs (second-order dynamics with linear drag), as described in Sec.~\ref{sec:lqr_drag}. Spatial competition is represented by the resulting CIV regions, while gradients indicate how each agent can adjust its state to improve team performance. A Gaussian utility density centered near the penalty spot emphasizes the tactical relevance (e.g., probability of scoring) of that area. Fig.~\ref{fig:casestudy} illustrates three representative situations with parameters listed in Table~\ref{tab:params}:

\begin{enumerate}[(a)]
\item When the blue defender is weaker (larger drag and control weights $(a,r)$), its CIV region $G_i$ contracts. The red attacker gradients point toward the high-value zone while also reducing the defender’s region of influence.
\item When these weights are reversed, the blue defender becomes stronger and its CIV region expands, causing the attackers’ gradients to favor positions on the flanks rather than approaching the defender.
\item With the same dynamics as in (b) but with the defender shifted upward, the boundaries become asymmetric: the upper striker moves outward to create separation while the lower striker moves toward the high-value region near the penalty spot.
\end{enumerate}

The simulation details were as follows.
The half-pitch has length $52.5\,\mathrm{m}$ and width $68\,\mathrm{m}$, including all standard markings. All positions were expressed in generic distance units, where $1$ unit corresponds to $5.25\,\mathrm{m}$.
A Gaussian utility density field $\phi(q)=\exp(-\|q-(2.9,0)\|^2/(2\cdot2^2))$ was used to model the tactical importance of the penalty spot.  
All agents had initial velocities $(1,0)$ in normalized units 
(about $19\,\mathrm{km/h}$, comparable to a player’s running speed). The drag and control-weight parameters $(a,r)$ for each scenario are summarized in Table~\ref{tab:params}.
The grid used was of size $350\times350$. All computations were done on an Intel\textsuperscript{\textregistered} Core\textsuperscript{TM} i7-9750H @ $2.60$\,GHz with $32$\,GB RAM.
 
\begin{table}
\vspace{4pt} 
\centering
\caption{LQR parameters $(a, r)$ for each configuration in Fig.~\ref{fig:casestudy}.}
\label{tab:params}
\renewcommand{\arraystretch}{1.1}
\setlength{\tabcolsep}{6pt}
\begin{tabular}{lcc}
\hline
\textbf{Scenario} & \textbf{Red agents $(a, r)$} & \textbf{Blue agent $(a, r)$} \\ 
\hline
(a) Striker advantage   & $(1.0,\,1.0)$ & $(1.5,\,1.5)$ \\ 
(b) Defender advantage  & $(1.5,\,1.5)$ & $(1.0,\,1.0)$ \\ 
(c) Asymmetric offset   & $(1.5,\,1.5)$ & $(1.0,\,1.0)$ \\ 
\hline
\end{tabular}
\end{table}

\section{CONCLUSION}
In this paper we introduced Cost-Induced Voronoi (CIV) regions and showed how they can be used to define agent and team utilities. Using the Reynolds Transport Theorem (RTT), we derived boundary-integral expressions for the gradients of these utilities. 

\addtolength{\textheight}{-12cm}   


\bibliographystyle{IEEEtran}
\bibliography{MyLibraryPetter}

\end{document}